\documentclass[10pt]{IEEEtran}

\usepackage{url,amsmath,amsfonts,amssymb,bm,dsfont,graphicx,tikz,forest,subfig,bbm,pgfplots,amsthm,float}
\usepackage{mathrsfs}
\usepackage{array}

\newtheorem{definition}{Definition}[section]

\newtheorem{theorem}{Theorem}[section]
\newtheorem{lemma}{Lemma}[section]
\newtheorem{corollary}{Corollary}[section]

\newtheorem{assumption}{Assumption}

\newcommand{\R}{{\mathbb R}}

\newcommand{\N}{{\mathbb N}}
\newcommand{\E}{{\mathbb E}}

\newcommand{\Py}{{\mathbb P}}

\newcommand{\vct}[1]{\bm{#1}}

\newcommand{\rv}[1]{\bm{\mathrm{#1}}}

\begin{document}

\title{{Sequential Experiment Design for Hypothesis Verification}}
\author{Dhruva Kartik, Ashutosh Nayyar and Urbashi Mitra
\thanks{
D. Kartik, A. Nayyar and U. Mitra are with the Department of Electrical
Engineering, University of Southern California, Los Angeles, CA 90089
(e-mail: mokhasun@usc.edu; ashutosn@usc.edu; ubli@usc.edu). This research was supported, in part, by National Science Foundation under Grant NSF CNS-1213128, CCF-1410009, CPS-1446901, Grant ONR N00014-15-1-2550, and Grant AFOSR FA9550-12-1-0215.
}}

\maketitle

\begin{abstract}
Hypothesis testing is an important problem with applications in target localization, clinical trials etc. Many active hypothesis testing strategies operate in two phases: an exploration phase and a verification phase. In the exploration phase, selection of experiments is such that a moderate level of confidence on the true hypothesis is achieved. Subsequent experiment design aims at improving the confidence level on this hypothesis to the desired level. In this paper, the focus is on the verification phase. A confidence measure is defined and active hypothesis testing is formulated as a confidence maximization problem in an infinite-horizon average-reward Partially Observable Markov Decision Process (POMDP) setting. The problem of maximizing confidence conditioned on a particular hypothesis is referred to as the hypothesis verification problem. The relationship between hypothesis testing and verification problems is established. The verification problem can be formulated as a Markov Decision Process (MDP). Optimal solutions for the verification MDP are characterized and a simple heuristic adaptive strategy for verification is proposed based on a zero-sum game interpretation of Kullback-Leibler divergences. It is demonstrated through numerical experiments that the heuristic performs better in some scenarios compared to existing methods in literature.

\end{abstract}

\section{Introduction}
Hypothesis testing is a classical problem and has been addressed in various settings. The problem can be described qualitatively as follows. An agent is interested in a phenomenon, and wants to test if the phenomenon conforms to any one of the hypotheses from a known class. The agent can perform various experiments and based on the observations from these experiments, it needs to infer the true hypothesis. As opposed to the one-shot hypothesis testing problem, an active agent can choose which experiment to perform based on the observations made in the past. The agent seeks to select experiments such that \emph{all} false hypotheses are eliminated as quickly as possible. 

Many active hypothesis testing strategies \cite{nitinawarat2013controlled,naghshvar2013active} operate in two phases. The first phase is an \emph{exploration} phase in which the experiment design is such that a moderate level of confidence is achieved on the true hypothesis. In most cases, this phase terminates in finite time almost surely \cite{chernoff1959sequential}. The second is a \emph{verification} phase in which the agent has a moderate level of confidence on some hypothesis and experiments are selected such that confidence on this hypothesis is improved to the desired level. When the desired confidence level is very high, the verification cost dominates the performance. In this paper, we make the notions of exploration and verification more formal and focus on analyzing the verification phase. 

Active hypothesis testing finds applications in many areas such as sensor selection for target detection and localization, state tracking, design of clinical trials and learning unknown functions from queries \cite{naghshvar2015bayesian}. Consequently, the verification phase plays an important role in all these applications. 

We consider a slightly different mathematical formulation for hypothesis testing than previously explored \cite{nitinawarat2013controlled,naghshvar2013active}. Using posterior belief on the set of hypotheses, we define a \emph{confidence} level called Bayesian log-likelihood ratio. The objective is to design an experiment selection strategy that maximizes the expected rate of increase in the confidence level. Our contributions in this paper can be summarized as follows:
\begin{enumerate}
\item We formulate the verification problem as an infinite-horizon average-reward Markov Decision Process (MDP) problem.
\item We characterize the optimal rate using infinite-horizon Dynamic Programming (DP).
\item We identify a set of \emph{critical} experiments. We then show that any strategy that selects these experiments while satisfying a stability criterion is asymptotically optimal.
\item We design a new heuristic experiment selection strategy and numerically show that it achieves better performance compared to existing methods in some scenarios.
\end{enumerate}

The rest of the paper is organized as follows. In Section \ref{priorwork}, we discuss the relation between our problem and those in prior works. Section \ref{sec:formulation} formulates the problem. Section \ref{sec:mdp} relates the problem to the MDP framework and defines critical experiments. In Section \ref{sec:solution}, we solve the DP and in Section \ref{sec:numerical}, we describe an adaptive strategy and numerically compare it with existing policies. We conclude the paper in Section \ref{sec:conclusion}.

\subsection{Prior Work}\label{priorwork}
The simplest active hypothesis testing problem was first formulated by Chernoff in \cite{chernoff1959sequential} inspired by Wald's analysis of the sequential probability ratio test \cite{wald1973sequential}. Thereafter, it has been generalized in different ways depending on the target application \cite{nitinawarat2013controlled,naghshvar2013active}. A major difference between our formulation and the formulation in these works is the reward structure. Prior works consider a combination of expected stopping time and Bayesian error probability. Fixed horizon problems have also been considered and they try to minimize the Bayesian error probability or maximal error probability \cite{nitinawarat2013controlled}. We define a notion of confidence and maximize the expected rate of increase in confidence over long horizons. In prior formulations, if the agent makes an error in guessing the true hypothesis, it incurs a cost of 1 (or some constant $c$) irrespective of its confidence level. Whereas in our formulation, we reward the agent for generating observations that result in a high confidence level on the true hypothesis. We believe that our formulation is related to the stopping time formulation because of the strong similarity in the results. In \cite{bessler1960theory,chernoff1959sequential,nitinawarat2013controlled,naghshvar2013active}, the authors obtain asymptotically tight performance bounds and design policies that are asymptotically optimal. When the policies in these works are adapted to the verification problem defined herein, they turn out to be open-loop and randomized. A closed loop policy was designed in \cite{naghshvar2012extrinsic} but this may not always be asymptotically optimal. In this paper, we design a strategy for verification that is more adaptive and conjecture that it is asymptotically optimal.

\subsection{Notation}\label{notation}
Random variables/vectors are denoted by upper case boldface letters, their realization by the corresponding lower case letter. We use calligraphic fonts to denote sets (e.g. $\mathcal{U}$) and $\Delta \mathcal{U}$ is the probability simplex over a finite set $\mathcal{U}$. In general, subscripts are used as time index. There are two exceptions (${\rho}_j(n),\rv{X}_j(n)$) to this convention where the subscript denotes the hypothesis and $n$ denotes time. For time indices $n_1\leq n_2$, $\rv{Y}_{n_1:n_2}$ is the short hand notation for the variables $(\rv{Y}_{n_1},\rv{Y}_{n_1+1},...,\rv{Y}_{n_2})$.
For a strategy $g$, we use $\Py^g[\cdot]$ and $\E^g[\cdot]$ to indicate that the probability and expectation depend on the choice of $g$. The Shannon entropy of a discrete distribution $p$ over a finite space $\mathcal{Y}$ is given by
\begin{equation}
H(p) = -\sum_{y \in \mathcal{Y}}p(y)\log p(y).
\end{equation}
And the Kullback-Leibler divergence between distributions $p$ and $q$ is given by
\begin{equation}
D(p || q) = \sum_{y \in \mathcal{Y}}p(y)\log\frac{p(y)}{q(y)}.
\end{equation}


\section{Problem Formulation}\label{sec:formulation}

\begin{figure}
\centering
\scalebox{1}{
\begin{forest}
[$\vct{\rho}(1)$
[${h=1}$, edge label={node[midway,above left,font=\scriptsize]{Nature}} 
  [$u_1$,edge label={node[midway,above left,font=\scriptsize]{Agent}} 
   [$y_1$,edge label={node[midway,above left,font=\scriptsize]{Nature}}
   	[$u_1$,edge label={node[midway,above left,font=\scriptsize]{Agent}}[$\vdots$]]
	[$u_2$[$\vdots$]]]
   [$y_2$
   	[$u_1$ [$\vdots$]]
	[$u_2$[$\vdots$]]]
  ]
  [$u_2$
  [$\vdots$] [$\vdots$]]
]
[${h=2}$ 
  [$u_1$
  [$\vdots$] [$\vdots$]]
    [$u_2$, 
   [$y_1$,
   	[$u_1$,[$\vdots$]]
	[$u_2$[$\vdots$]]]
   [$y_2$
   	[$u_1$ [$\vdots$]]
	[$u_2$[$\vdots$]]]
  ]
]
]
\end{forest}}
\caption{Agent's choices and subsequent observations represented as a tree. Every instance of the probability space can be uniquely represented by a path in this tree.}
\label{fig:illustree}
\end{figure}
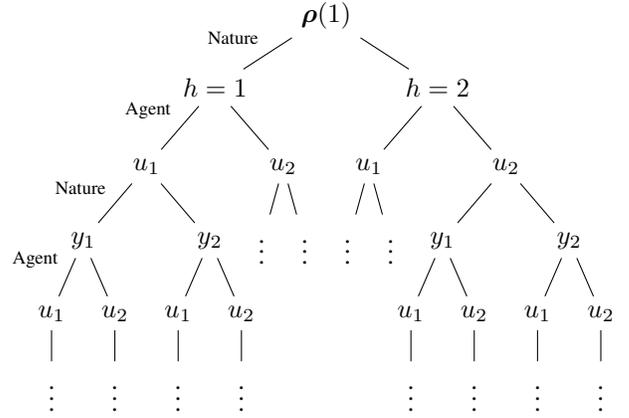

Let $\mathcal{H} \subset \N$ be a finite set of hypotheses and let $\rv{H}$ be the true hypothesis. At each time $n \in \N$, the agent can perform an experiment $\rv{U }_n \in \mathcal{U}$ and obtain an observation $\rv{Y}_n \in \mathcal{Y}$. For simplicity, let us also assume that the sets $\mathcal{U}$ and $\mathcal{Y}$ are finite. When an experiment $u \in \mathcal{U}$ is performed for the $k$th time, the observation $\rv{Y}$ obtained is given by
\begin{equation}
\rv{Y} = \xi(\rv{H}, u,\rv{W}_k^u),
\end{equation}
where $\{\rv{W}_k^u: u \in \mathcal{U}, k \in \N\}$ is a collection of mutually independent and identically distributed primitive random variables. The observation $\rv{Y}_n$ at time $n$ can be expressed as
\begin{equation}
\rv{Y}_n = \xi(\rv{H}, \rv{U}_n,\rv{W}_n).
\end{equation}
The probability of observing $y$ after performing an experiment $u$ under hypothesis $h$ is denoted by $p_h^u(y)$.

The information available at time $n$, denoted by $\rv{I}_n$, is the collection of all experiments performed and the corresponding observations up to time $n-1$, i.e.
\begin{equation}
\rv{I}_n = \{\rv{U}_{1:n-1},\rv{Y}_{1:n-1}\}.
\end{equation}
Actions of the agent at time $n$ can be functions of $\rv{I}_n$. Let the policy used for selecting the experiment be $g_n$, i.e.
\begin{equation}
\rv{U}_n = g_n(\rv{I}_n).
\end{equation}
The sequence of all the policies $\{g_n\}$ is denoted by $g$ which is referred to as a \emph{strategy}. Let the collection of all such strategies be $\mathcal{G}$. 

Using the available information, the agent forms a \emph{posterior belief} $\vct{\rho}(n)$ on $\rv{H}$ at time $n$ which is given by
\begin{equation}\label{postbelief}
\rho_h(n) = \Py[\rv{H} = h \mid \rv{Y}_{1:n-1},\rv{U}_{1:n-1}].
\end{equation}

\begin{definition}[Bayesian Log-Likelihood Ratio]
The Bayesian log-likelihood ratio $\mathcal{C}_h(\vct{\rho})$ associated with an hypothesis $h \in \mathcal{H}$ is defined as
\begin{equation}
\mathcal{C}_h(\vct{\rho}) := \log\frac{\rho_h}{1-\rho_h}.\\
\end{equation}
\end{definition}
\vspace{0.05in}
The Bayesian log-likelihood ratio (BLLR) is the logarithm of the ratio of the probability that hypothesis $h$ is true versus the probability that hypothesis $h$ is not true. BLLR is obtained by applying the \emph{logit} function (also referred to as \emph{log-odds} in statistics \cite{hosmer2013applied}) on the posterior belief $\rho_h$. The logit function amplifies increments in $\rho_h$ when $\rho_h$ is close to $0$ or $1$. We can interpret BLLR as a measure of confidence on hypothesis $h$ and thus, we refer to it as \emph{confidence level}.

\begin{figure}[]

\hspace{0.0in}
\begin{tikzpicture}
\begin{axis}[xlabel = $p$, xmajorgrids = true,ymajorgrids = true,grid style = dashed,legend pos=north west]
\addplot[domain = 0.0001:0.9999,samples = 1000]{ln(x/(1-x))};
\legend{$\log\frac{p}{1-p}$}
\end{axis}
\end{tikzpicture}
\caption{The logit function is the inverse of the logistic sigmoid function $1/(1+e^{-x})$. It is widely used in statistics and machine learning to quantify confidence level \cite{hosmer2013applied}.}
\label{logitplot}
\end{figure}
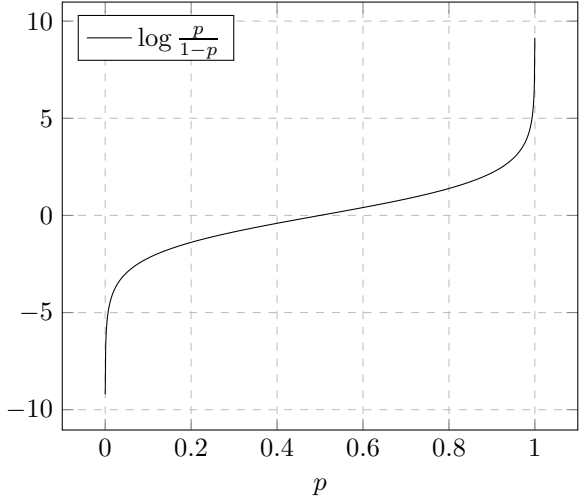

The objective is to design an experiment selection strategy $g$ such that the confidence level $\mathcal{C}_{\rv{H}}$ on the true hypothesis $\rv{H}$ increases as quickly as possible. In other words, the total reward after acquiring $N$ observations is the average rate of increase in the confidence level on the true hypothesis $\rv{H}$ and is given by
\begin{equation}
\frac{\mathcal{C}_{\rv{H}}(\vct{\rho}(N+1))-\mathcal{C}_{\rv{H}}(\vct{\rho}(1))}{N}.
\end{equation}
More explicitly, we seek to design a strategy $g$ that maximizes the asymptotic expected reward $K(g)$ which is defined as
\begin{align*}
K(g) &:= \lim_{N \to \infty} \inf \frac{1}{N} \;\E^g \left[\mathcal{C}_{\rv{H}}(\vct{\rho}(N+1))- \mathcal{C}_{\rv{H}}(\vct{\rho}(1)) \right].
\end{align*}
Henceforth, we refer to this problem as the \emph{Expected Confidence Maximization} (ECM) problem for \emph{hypothesis testing}. For a hypothesis $h$ and a strategy $g \in \mathcal{G}$, define $J(g,h)$ as
\begin{align*}
&\lim_{N \to \infty} \inf \frac{1}{N} \;\E^{g} \left[\mathcal{C}_{\rv{H}}(\vct{\rho}(N+1))- \mathcal{C}_{\rv{H}}(\vct{\rho}(1)) \mid \rv{H} = h\right].
\end{align*}
The value $J(g,h)$ represents the performance of a strategy $g$ conditioned on the hypothesis $h$. Let \begin{equation}\label{supdef}
    J^*(h) = \sup_{g\in \mathcal{G}}J(g,h).
\end{equation}
For a given hypothesis $h$, we refer to the problem of maximizing $J(g,h)$ as the \emph{hypothesis verification} problem. Let $g^*(h)$ be an optimal verification strategy, i.e. it achieves the supremum in equation (\ref{supdef}). We will later show that the existence of an optimal strategy $g^*(h)$ is guaranteed under a mild assumption.

\subsection{Hypothesis Testing vs Hypothesis Verification}
The optimal verification cost $J^*(h)$ can be used to obtain an upper bound on the expected reward $K(g)$ in the hypothesis testing problem.
\begin{lemma}\label{veriflemma}
For any experiment selection strategy $g \in \mathcal{G}$, we have
\begin{equation}
    K(g) \leq \sum_{h\in\mathcal{H}}\rho_h(1)J^*(h).
\end{equation}
\end{lemma}
\begin{proof}
For any strategy $g \in \mathcal{G}$, we have
\begin{align}
    K(g) &= \sum_{h\in\mathcal{H}}\rho_h(1)J(g,h) \leq \sum_{h\in\mathcal{H}}\rho_h(1)J^*(h).
\end{align}
The last inequality follows from the definition of $J^*(h)$.
\end{proof}
It is clear from the proof of Lemma \ref{veriflemma} that this upper bound is achieved by employing the strategy $g^*(h)$ when hypothesis $h$ is true. However, the agent cannot use different strategies under different hypotheses because it does not know the true hypothesis $\rv{H}$. Therefore, we propose an experiment selection strategy of the following form. Similar strategies have also been used in \cite{naghshvar2013active}.
\begin{equation}
\bar{g}(\vct{\rho}) = 
\begin{cases}
g^*(h)(\vct{\rho}) &\text{if for some } h, \rho_h > \bar{\rho}\\
g^e(\vct{\rho}) &\text{otherwise},
\end{cases}
\end{equation}
where $0.5 < \bar{\rho} < 1$ is a constant and $g^e$ is an \emph{exploration} strategy. The interpretation of the strategy $\bar{g}$ is that when the agent has a moderate level of confidence on some hypothesis $h$, it employs the corresponding verification strategy $g^*(h)$. This is to \emph{verify} if hypothesis $h$ is indeed true by further improving its confidence level. When the agent is \emph{not} very confident about any particular hypothesis, the agent employs an exploration strategy $g^e$. The primary purpose of the exploration strategy is to ensure that $\rho_{\rv{H}}$ eventually crosses the threshold $\bar{\rho}$. A naive exploration strategy is to randomly select every experiment uniformly. Better exploration strategies do exist \cite{naghshvar2013active, naghshvar2012extrinsic}. It remains to show that a strategy like $\bar{g}$ can indeed achieve the upper bound in Lemma \ref{veriflemma}. In this paper, we focus on the hypothesis verification problem. We derive sufficient conditions for an experiment selection strategy to be an \emph{optimal} verification strategy.

\section{Markov Decision Process Formulation}\label{sec:mdp}
In this section, we show that the verification problem can be formulated as an infinite-horizon average-reward MDP problem. All of the following analysis is for $h = 1$ and with slight abuse of notation, we henceforth refer to $g^*(1)$ and $J(g,1)$ as $g^*$ and $J(g)$, respectively. The same analysis can be repeated for any other $h$ to obtain similar results.

The state of the MDP is the posterior belief $\vct{\rho}(n)$. The posterior belief is updated using Bayes' rule. Thus, if $\rv{U}_n = u$ and $\rv{Y}_n = y$, we have
\begin{align}
{\rho}_{h}(n+1) = \frac{\rho_h(n)p_h^u(y)}{\sum_{h'}\rho_{h'}(n)p^u_{h'}(y)}.\label{bayes}
\end{align} 
For convenience, we denote the Bayes' update in (\ref{bayes}) by
\begin{align}
\vct{\rho}({n+1}) &= F(\vct{\rho}(n),\rv{U}_n,\rv{Y}_{n}).
\end{align}
Since $\rv{H} = 1$, we have $\rv{Y}_n = \xi(1,\rv{U}_n,\rv{W}_n)$. Clearly, the dynamics of this system are Markovian. The expectation of average confidence rate under a strategy ${g}$ is given by
\begin{align}
\label{objfin}J_N(g) :&=\frac{1}{N}\E^{{g}} \left[\mathcal{C}_{1}(\vct{\rho}(N+1))- \mathcal{C}_{1}(\vct{\rho}(1)) \right]\\
&= \frac{1}{N}\E^{{g}} \sum_{n = 1}^N \left[\mathcal{C}_1(\vct{\rho}(n+1))- \mathcal{C}_1(\vct{\rho}(n)) \right]\\
\nonumber&= \frac{1}{N}\E^{{g}} \sum_{n = 1}^N \E \left[\mathcal{C}_{1}(\vct{\rho}(n+1))- \mathcal{C}_{1}(\vct{\rho}(n)) \mid \rv{I}_n,\rv{U}_n\right]\\
\nonumber &= \frac{1}{N}\E^{{g}} \sum_{n = 1}^N \E \left[\mathcal{C}_{1}(\vct{\rho}(n+1))- \mathcal{C}_{1}(\vct{\rho}(n)) \mid \vct{\rho}(n),\rv{U}_n\right]\\
&=: \frac{1}{N}\E^{{g}} \sum_{n = 1}^N r(\vct{\rho}(n),\rv{U}_n).
\end{align}
Instantaneous reward for this MDP is the expected instantaneous increase in the confidence level and is given by
\begin{align}
\nonumber r(\vct{\rho},u) &=  \sum_{y\in \mathcal{Y}}p_1^u(y)\log{\frac{\rho_1p_1^u(y)}{\sum_{j\neq1}\rho_jp_j^u(y)}} - \log\frac{\rho_1}{(1-\rho_1)}\\
&= \sum_{y\in \mathcal{Y}}p_1^u(y)\log{\frac{p_1^u(y)}{\sum_{j\neq1}\tilde{\rho}_jp_j^u(y)}},
\end{align}
where $\tilde{\rho}_j = \rho_j/(1-\rho_1)$. Note that $\tilde{\rho}_j$ is a probability distribution over the set of alternate hypotheses $\tilde{\mathcal{H}} = \mathcal{H}\setminus\{1\}$. Also, notice that $r(\vct{\rho},u)$ is a KL-divergence between two distributions and hence, is always non-negative. The objective is to find a strategy $g^*$ that maximizes the following average reward
\begin{align}
J(g) := \lim_{N \to \infty}\inf\frac{1}{N}\sum_{n=1}^N \E^g(r(\vct{\rho}(n),\rv{U}_n)).
\end{align}
We use Dynamic Programming (DP) to characterize optimal solutions for this infinite-horizon problem. In this framework, it can be shown that the randomized strategies used in \cite{chernoff1959sequential,nitinawarat2013controlled,naghshvar2013active} asymptotically achieve optimal rate $J^*$. Additionally, we identify a class of strategies that also achieve optimal rate and possibly, converge faster to the optimal rate than policies used in prior works.

Consider the following fixed point equation for the infinite horizon MDP
\begin{equation}\label{inffp}
J' + w(\vct{\rho}) = \max_{u}\{r(\vct{\rho},u) + \sum_y p_1^u(y) w(F(\vct{\rho},u,y))\},
\end{equation}
where $J' \in \R$ is some constant and $w:\Delta{\mathcal{H}} \rightarrow \R$ is some mapping. If such $J'$ and $w$ exist, then with some algebra (see \cite{kumar2015stochastic} for details), we can conclude the following for any experiment selection strategy $g$ (possibly non-stationary)
\begin{align}
&\lim_{N \to \infty}\sup\frac{1}{N}\sum_{n=1}^N \E^g(r(\vct{\rho}(n),\rv{U}_n)) \\
\leq &\lim_{N \to \infty}\sup\frac{1}{N}(\E^gw(\vct{\rho}(1)) - \E^gw(\vct{\rho}(N+1))) + J'.
\end{align}
If we can show that 
\begin{equation}\label{limit}
\lim_{N \to \infty}\sup\frac{1}{N}\left(\E^gw(\vct{\rho}(1)) - \E^gw(\vct{\rho}(N+1))\right) \leq 0,
\end{equation}
for every strategy $g$, then clearly the optimal rate $J^* \leq J'$. Additionally, if for some strategy $g^*$, 
\begin{equation}\label{limit2}
\lim_{N \to \infty}\inf\frac{1}{N}\left(\E^{g^*}w(\vct{\rho}(1)) - \E^{g^*}w(\vct{\rho}(N+1))\right) = 0
\end{equation}
is satisfied and the experiment selected by $g^*$ is a maximizer in the fixed point equation (\ref{inffp}), then $g^*$ is indeed an optimal strategy and $J^* = J'$ \cite{kumar2015stochastic}. Our objective now is to find $J'$ and a function $w$ that satisfy these conditions. We make the following assumption on the conditional distributions $p_h^u(y)$.

\begin{assumption}\label{boundedassump}
There exists a constant $B>0$ such that $|\lambda_j^i(u,y)| < B$ for every experiment $u$, observation $y$ and hypotheses $i,j \in \mathcal{H}$, where
$$
\lambda_j^i(u,y) := \log\frac{p_i^u(y)}{p_j^u(y)}.
$$
\end{assumption}

We use the following defined quantities throughout our proofs. Let
\begin{align}
\vct{\alpha}^* &:= \arg \max_{\vct{\alpha} \in \Delta\mathcal{U}} \min_{j\neq 1} \sum_{u}\alpha_u D(p_1^u || p_j^u)\\
\vct{\beta}^* &:= \arg \min_{\vct{\beta} \in \Delta\tilde{\mathcal{H}}} \max_{u \in \mathcal{U}} \sum_{j \neq 1}\beta_j D(p_1^u || p_j^u).
\end{align}
Since the sets $\mathcal{U}$ and $\mathcal{H}$ are finite, existence of $\vct{\alpha}^*$ and $\vct{\beta}^*$ is guaranteed and also, by minimax theorem \cite{osborne1994course}
\begin{align}
\nonumber \max_{\vct{\alpha} \in \Delta\mathcal{U}} \min_{j\neq 1} \sum_{u}\alpha_u D(p_1^u || p_j^u) &= \min_{\vct{\beta} \in \Delta\tilde{\mathcal{H}}} \max_{u \in \mathcal{U}} \sum_{j \neq 1}\beta_j D(p_1^u || p_j^u)\\
& =: R^*. 
\end{align}
We refer to the elements in the support of $\vct{\beta}^*$ as \emph{critical hypotheses} and those in the support of $\vct{\alpha}^*$ as \emph{critical experiments}.
In particular, we show that the optimal rate $J^* = R^*$.

\section{Dynamic Programming Solution }\label{sec:solution}

In this section,  we solve the MDP formulated in Section \ref{sec:mdp}. Lemma \ref{lemma:fp} identifies a solution for the fixed point equation (\ref{inffp}) and the subsequent Corollary \ref{corol} is used to obtain an upper bound on $J^*$. We then show that this upper bound can indeed be achieved.
\begin{lemma}\label{lemma:fp}
The fixed point equation (\ref{inffp}) is satisfied with $J' = R^*$ and
\begin{equation}
w(\vct{\rho}) = -\sum_{j \neq 1}\beta^*_j\log\frac{\rho_j}{1-\rho_1} = -\sum_{j \neq 1}\beta^*_j\log\tilde{\rho}_j.
\end{equation}
Also, any critical experiment is a maximizer in the fixed point equation (\ref{inffp}).
\end{lemma}
\begin{proof}
Define
$
v(\vct{\rho}) := w(\vct{\rho}) + \mathcal{C}_1(\vct{\rho})$, that is $$v(\vct{\rho}) := \sum_{j \neq 1}\beta^*_j\log\frac{\rho_1}{\rho_j}.$$
Therefore, we have for every $u$
\begin{align}
& \sum_y p_1^u(y) w(F(\vct{\rho},u,y)) - w(\vct{\rho}) \\
& \qquad\qquad =\sum_y p_1^u(y)v(F(\vct{\rho},u,y)) - v(\vct{\rho}) - r(\vct{\rho},u).
\end{align}
This is because $r(\vct{\rho},u)$ equal to the expected increase in the confidence level $\mathcal{C}_1(\vct{\rho})$ after performing the experiment $u$. Hence,
\begin{align}
&\max_{u}\{r(\vct{\rho},u) + \sum_y p_1^u(y) w(F(\vct{\rho},u,y))\} - w(\vct{\rho})\\
&= \max_{u} \sum_y p_1^u(y)v(F(\vct{\rho},u,y)) - v(\vct{\rho}) \\
&= \max_{u} \sum_y p_1^u(y)\sum_{j\neq1} \beta^*_j\log\frac{\rho_1p_1^u(y)}{\rho_j p_j^u(y)} - v(\vct{\rho})\\
&=\max_{u} \sum_y p_1^u(y)\sum_{j\neq1} \beta^*_j(\log\frac{\rho_1}{\rho_j}+\log\frac{p_1^u(y)}{p_j^u(y)}) - v(\vct{\rho})\\
&= \max_{u} \sum_y p_1^u(y)\sum_{j\neq1} \beta^*_j\log\frac{p_1^u(y)}{p_j^u(y)} + v(\vct{\rho}) - v(\vct{\rho})\\
&= \label{maximizer}\max_{u} \sum_{j\neq1} \beta^*_jD(p_1^u || p_j^u) = R^* = J'.
\end{align}
The last equality follows from the fact that $\vct{\beta}^*$ is a solution for the minimax problem and the minimax value is equal to $R^*$. Therefore, $J'$ and $w$ satisfy the fixed point equation (\ref{inffp}). Note that any critical experiment $u$ is a maximizer in (\ref{maximizer}).\end{proof}
%

\begin{corollary}\label{corol}
For any strategy $g$, we have
\begin{align}
&\lim_{N \to \infty}\sup\frac{1}{N}\left(\E^g w[\vct{\rho}(1)) - \E^gw(\vct{\rho}(N+1)]\right)\\
=& \lim_{N \to \infty}\sup\frac{1}{N}\sum_{j \neq 1}\beta^*_j\E^g\log{\tilde{\rho}_j(N+1)} \leq 0.
\end{align}
\end{corollary}
\begin{proof}
This is simply because $\tilde{\rho}_j(N+1) \leq 1$.
\end{proof}

\begin{theorem} The optimal average rate $J^* \leq R^*.$
\end{theorem}
\begin{proof}
This directly follows from the fact that $w$ defined in Lemma \ref{lemma:fp} satisfies inequality (\ref{limit}) and with $J' = R^*$, the fixed point equation (\ref{inffp}) is satisfied.
\end{proof}

\begin{theorem}\label{achievethm}
The optimal average rate $J^* = R^*$.
\end{theorem}
\begin{proof}
It is sufficient to show that there exists a strategy $g^*$ that satisfies
\begin{equation}\label{stability}
\lim_{N \to \infty}\inf\frac{1}{N}\sum_{j \neq 1}\beta^*_j\E^{g^*}\log\tilde{\rho}_j(N+1) = 0,
\end{equation}
and the strategy $g^*$ selects only critical experiments. Let
\begin{align}
\rv{X}_j(n+1) = \rv{X}_{j}(n) + \lambda^j_1(\rv{U}_n,\rv{Y}_n),
\end{align}
where $\rv{X}_j(1) = \log \rho_j(1)$. If $\rv{X}_j(N+1) = x_j$ and $\tilde{\rho}_j(N+1)= \tilde{\rho}_j$, we have
\begin{equation}
\log \tilde{\rho}_j = x_j - \log\sum_{k\neq1} e^{x_k}.
\end{equation}
Consider an open-loop randomized strategy where at each time, the experiment is selected independently using the distribution $\vct{\alpha}^*$. Clearly, this strategy selects only critical experiments. Under this open-loop strategy, we have for any $j\neq 1$
\begin{align}
\E [\lambda_1^j(\rv{U},\rv{Y})] &= \sum_u \alpha^*_u \sum_y p_1^u(y)\log(p_j^u(y)/p_1^u(y))\\
&= \sum_u -\alpha^*_uD(p_1^u||p_j^u) =: -R_j. 
\end{align}
Notice that for \emph{every} critical hypothesis $j$, $R_j = R^*$ and for every non-critical alternate hypothesis, $R_j > R^*$. This follows from the definition of $\vct{\alpha}^*$.
Further, we have
\begin{align}
\frac{1}{N}\E\rv{X}_j(N+1) &= \frac{1}{N}\E \rv{X}_j(0) -R_j.
\end{align}
As $N \to \infty$, the term $\rv{X}_j(0)/N \to 0$ and we can ignore it. Thus, for every critical hypothesis $j$,
\begin{align*}
    \frac{1}{N}\E\log\tilde{\rho}_j(N+1) &= \frac{1}{N}\E[\rv{X}_j(N+1) - \log\sum_{k\neq1} e^{\rv{X}_k(N+1)}]\\
    &= -R^* - \frac{1}{N}\E\log\sum_{k\neq1} e^{\rv{X}_k(N+1)}.
\end{align*}
We can ignore the non-critical hypotheses because $\beta_j^* = 0$ for non-critical hypotheses. If we can show that the second term approaches $-R^*$ as $N\to\infty$, then clearly, the condition (\ref{stability}) is satisfied with equality. Using Strong Law of Large Numbers (SLLN) \cite{durrett2010probability}, we can conclude that for \emph{every} alternate hypothesis $j$,
\begin{align}
    \frac{1}{N}\rv{X}_j(N+1) \to -R_j,
\end{align}
with probability 1. We can use SLLN because of Assumption \ref{boundedassump}. Therefore,
\begin{align}
    \max_{j\neq 1}\{\frac{1}{N}\rv{X}_j(N+1)\} \to \max_{j\neq 1}\{-R_j\} = -R^*.
\end{align}
Further, because of Assumption \ref{boundedassump}, $\rv{X}_j(N+1)/N$ is uniformly bounded by $B$ for every alternate hypothesis $j$. Thus, using bounded convergence theorem \cite{durrett2010probability}, we have
\begin{align}
    \E\max_{j\neq 1}\{\frac{1}{N}\rv{X}_j(N+1)\} \to  -R^*.
\end{align}

For the log sum exponential function, we have the following
\begin{align}
\max_{j\neq1} \{\rv{X}_j(N+1)\}&\leq \log\sum_{k\neq 1}e^{\rv{X}_k(N+1)} \\
\nonumber&\leq \max_{j\neq1} \{\rv{X}_j(N+1)\} + \log |\mathcal{H}| -1.
\end{align}
Therefore,
\begin{align}
    \frac{1}{N}\E\log\sum_{k\neq 1}e^{\rv{X}_k(N+1)}  \to  -R^*.
\end{align}
Thus, the open-loop randomized policy $\vct{\alpha}^*$ is asymptotically optimal and $J^* = R^*$.
\end{proof}

To summarize, the following conditions are sufficient for a stationary verification strategy $g$ to be asymptotically optimal:
\begin{enumerate}
    \item The strategy $g$ only selects critical experiments, i.e. experiments from the support of $\vct{\alpha}^*$.
    \item The \emph{stability} criterion in (\ref{stability}) is satisfied, i.e.
    \begin{equation}
\lim_{N \to \infty}\inf\frac{1}{N}\sum_{j \neq 1}\beta^*_j\E^{g^*}\log\tilde{\rho}_j(N+1) = 0.
\end{equation}
\end{enumerate}
These conditions suggest that there could be many strategies other than the open-loop randomized strategy used in Theorem \ref{achievethm} that achieve asymptotic optimality.

\section{Numerical Results}\label{sec:numerical}

In this section, we propose a new heuristic based on a Kullback-Leibler divergence zero-sum game and demonstrate numerically that this heuristic's performance is close to the maximum achievable confidence rate $R^*$. We first briefly describe all the strategies used in our experiments.

\subsubsection{Extrinsic Jensen-Shannon (EJS) Divergence}
Extrinsic Jensen-Shannon divergence as a notion of information was first introduced in \cite{naghshvar2012extrinsic}. Using our notation, EJS for a query $u$ at some belief state $\vct{\rho}$ is given by
\begin{align}
EJS(\vct{\rho},u) = \E [\mathcal{C}(F(\vct{\rho},u,\rv{Y})) - \mathcal{C}(\vct{\rho})],
\end{align}
where 
\begin{equation}
    \mathcal{C}(\vct{\rho}) = \sum_{i \in \mathcal{H}}\rho_i\log\frac{\rho_i}{1-\rho_i} = \sum_{i \in \mathcal{H}}\rho_i\mathcal{C}_i(\vct{\rho}).
\end{equation}Notice that the only random variable in the expression above is $\rv{Y}$ and the expectation is with respect to the distribution $\sum_{h \in \mathcal{H}}\rho_hp^u_h(y)$ on $\mathcal{Y}$. The EJS heuristic selects the experiment $u$ that maximizes $EJS(\vct{\rho},u)$ for a given state $\vct{\rho}$.

\subsubsection{Open Loop Verification (OPE)}
As discussed earlier, the strategies in \cite{naghshvar2013active,nitinawarat2013controlled,chernoff1959sequential} when specialized to verification are open-loop and randomized. According to this strategy, the queries are randomly selected independently in an open-loop manner from the distribution $\vct{\alpha}^*$. Recall that this strategy is asymptotically optimal as shown in Theorem \ref{achievethm}.

\subsubsection{KL-divergence Zero-sum Game (KLZ)}\label{kzg}
We design the following heuristic. Consider a zero-sum game \cite{osborne1994course} in which the first player (maximizing) selects an experiment $u \in \mathcal{U}$ and the second player (minimizing) selects an alternate hypothesis $j \in \tilde{\mathcal{H}}$. The payoff for this zero-sum game is the KL-divergence $D(p_1^u || p_j^u)$. The agent picks an experiment $u$ that maximizes
$$\mathscr{P}(\vct{\rho},u) := {\sum_{j\neq i}\tilde{\rho}_jD(p_1^u||p_j^u)}.$$
This strategy can be interpreted as the first player's best-response when the second player uses the mixed strategy $\tilde{{\rho}}_j$ to select an alternate hypothesis. Note that the mixed strategy $\vct{\alpha}^*$ used in OPE is an equilibrium strategy for the maximizing player.

\subsection{Simulation Setup}
To simulate these heuristics, we first consider a simple setup with three hypotheses and two queries. The conditional distributions $p_i^u(y)$ for each of these queries are illustrated in Figure \ref{table1}.

\begin{figure}[h]
\centering
 \subfloat[][Query $u^1$]{
 \begin{tabular}{ | l | c | c |}
 \hline
   & $y = 0$& $y= 1$ \\
  \hline
  $h_0$ & 0.8 & 0.2 \\
  $h_1$ & 0.2 & 0.8 \\
  $h_2$ & 0.8 & 0.2\\
  \hline
\end{tabular}
 }
 \subfloat[][Query $u^2$]{
 \begin{tabular}{ | l | c | c |}
 \hline
   & $y = 0$& $y= 1$ \\
  \hline
  $h_0$ & 0.8 & 0.2 \\
  $h_1$ & 0.8 & 0.2 \\
  $h_2$ & 0.2 & 0.8\\
  \hline
\end{tabular}
 }
 \caption{Conditional distributions $p_i^u(y)$ for each query}
 \label{table1}
\end{figure}
The queries are designed such that when $\rv{H} = h_0$, the agent is forced to make both queries $u^1$ and $u^2$. This is because hypotheses $h_0$ and $h_2$ are indistinguishable under query $u^1$ and similarly, hypotheses $h_0$ and $h_1$ are indistinguishable under query $u^2$. We illustrate the evolution of expected confidence rate $J_N$ under hypothesis $h_0$ in Figure \ref{plot1}. The heuristics EJS and KLZ come very close to the maximum achievable rate. OPE eventually achieves maximal rate but very slowly.

\begin{figure}[]

\includegraphics[width=\columnwidth]{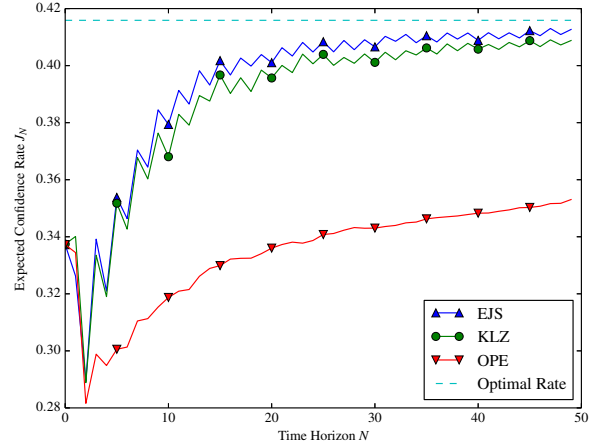}%

\caption{Evolution of expected confidence rate $J_N$ under hypothesis $h_0$ in the first setup with queries $u^1$ and $u^2$. Note the subpar performance of OPE in this setup.}
\label{plot1}
\end{figure}

\begin{figure}[H]
\centering
 \subfloat[][Query $u^3$]{
 \begin{tabular}{ | l | c | c |}
 \hline
   & $y = 0$& $y= 1$ \\
  \hline
  $h_0$ & 0.8 & 0.2 \\
  $h_1$ & $1-\delta$ & $\delta$ \\
  $h_2$ & 0.8 & 0.2\\
  \hline
\end{tabular}
 }
 \subfloat[][Query $u^4$]{
 \begin{tabular}{ | l | c | c |}
 \hline
   & $y = 0$& $y= 1$ \\
  \hline
  $h_0$ & 0.8 & 0.2 \\
  $h_1$ & 0.8 & 0.2 \\
  $h_2$ & $1-\delta$ & $\delta$\\
  \hline
\end{tabular}
 }
 \caption{Conditional distributions $p_i^u(y)$ for each additional query. Here, $\delta = 0.0000001$.}
 \label{table2}
\end{figure}

In the second experimental setup, we include two additional queries $u^3$ and $u^4$ characterized by the distributions in Figure \ref{table2}. When $\rv{H} = h_0$ the queries $u^3$ and $u^4$ together can eliminate at a much faster rate than $u^1$ and $u^2$. Intuitively, this is because when the agent performs $u^3$ and observes $y=1$, the belief on $h_1$ decreases drastically because $y=1$ is extremely unlikely under hypothesis $h_1$. Similarly, $u^4$ is very effective in eliminating $h_2$. The evolution of expected confidence rate under hypothesis $h_0$ with additional experiments $u^3$ and $u^4$ is shown in Figure \ref{plot2}. The heuristics KLZ and OPE select queries $u^3$ and $u^4$ under hypothesis $h_0$. But the greedy heuristic EJS usually selects only $u^1$ and $u^2$ and fails to realize that queries $u^3$ and $u^4$ are more effective under hypothesis $h_0$. The greedy EJS approach fails because queries $u^3$ and $u^4$ are constructed in such way that they are optimal over longer horizons but are sub-optimal over shorter horizons. Thus the assumption required for asymptotic optimality of EJS in \cite{naghshvar2012extrinsic} does not hold in this setup.
\begin{figure}[]

\includegraphics[width=\columnwidth]{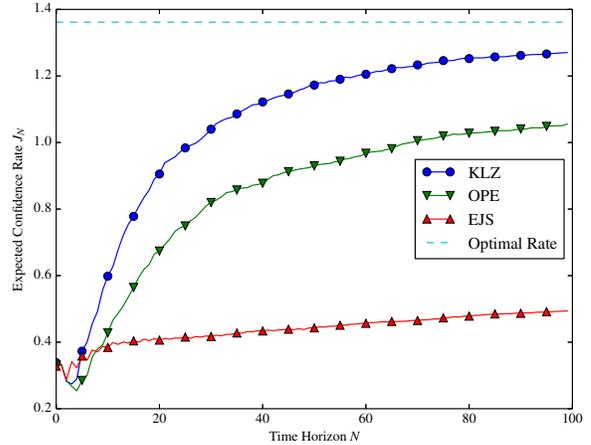}%

\caption{Evolution of expected confidence rate $R_N$ under hypothesis $h_0$ in the second setup with additional queries $u^3$ and $u^4$. Note the subpar performance of OPE and EJS in this setup.}
\label{plot2}
\end{figure}

\subsection{Stopping Time Formulation}
In \cite{chernoff1959sequential,nitinawarat2013controlled,naghshvar2013sequentiality}, a stopping time formulation for hypothesis testing is considered. The sampling process stops when the belief on some hypothesis exceeds a threshold or equivalently, when the confidence $\mathcal{C}_h(\vct{\rho}) > \log L$, where $L$ is a parameter. Let this stopping time be $\rv{N}$. Under this stopping criterion, we numerically study the expected stopping time for all the strategies discussed. The plots in Figures \ref{plot3} and \ref{plot4} depict the quantity $\E[\rv{N}]/\log L$ as a function of the parameter $L$. Numerical results suggest that our heuristic performs better even in the stopping time formulation.

\begin{figure}[]

\includegraphics[width=\columnwidth]{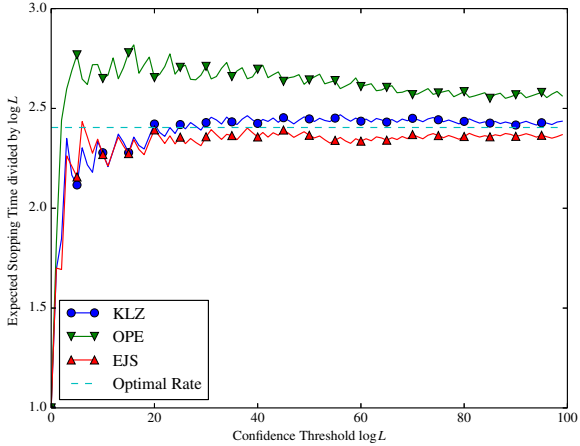}%

\caption{Evolution of expected stopping time under hypothesis $h_0$ in the first setup with queries $u^1$ and $u^2$. Note the subpar performance of OPE in this setup.}
\label{plot3}
\end{figure}

\begin{figure}[]

\includegraphics[width=\columnwidth]{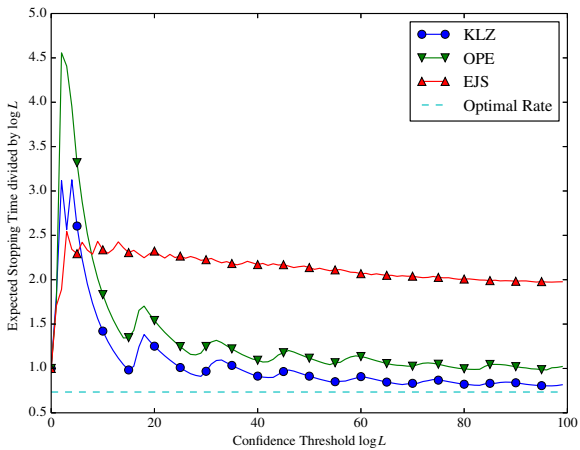}%

\caption{Evolution of expected stopping time under hypothesis $h_0$ in the second setup with additional queries $u^3$ and $u^4$. Note the subpar performance of OPE and EJS in this setup.}
\label{plot4}
\end{figure}

\section{Conclusion}\label{sec:conclusion}
In this paper, we formulate the problem of quickly verifying a given hypothesis using observations from experiments as an infinite horizon average cost MDP. We characterize the optimal rate of this MDP using infinite horizon dynamic programming. A stability criterion arises out of the DP equations. We show that any strategy that satisfies this stability criterion while selecting experiments from a critical set is asymptotically optimal. We proposed a heuristic adaptive strategy and numerically demonstrated that it performs better than open-loop policies in the non-asymptotic regime. For future work, we intend to use this stability criterion, perhaps with additional penalty terms, to design strategies with better non-asymptotic performance.

\bibliographystyle{IEEEbib}

\bibliography{refs}
\end{document}